\documentclass{article}

     \PassOptionsToPackage{numbers, compress}{natbib}

     \usepackage[preprint]{neurips_2020}

\usepackage[utf8]{inputenc} 
\usepackage[T1]{fontenc}    
\usepackage{hyperref}       
\usepackage{url}            
\usepackage{booktabs}       
\usepackage{amsfonts}       
\usepackage{nicefrac}       
\usepackage{microtype}      
\usepackage{graphicx}
\usepackage{amsmath}
\usepackage{amsthm}
\usepackage{mathtools}
\usepackage{caption}
\usepackage{subcaption}
\usepackage{tikz}
\usepackage{caption,subcaption}

\usepackage{wrapfig}

\newcommand*\circled[1]{\tikz[baseline=(char.base)]{
            \node[shape=circle,draw,inner sep=0.5pt,minimum size=0.15cm] (char) {#1};}}

\newtheorem{prop}{Proposition}
\newtheorem{theo}{Theorem}
\newcommand{\VAE}{\texttt{VAE}}
\newcommand{\VAEs}{\texttt{VAEs}}
\newcommand{\DGC}{\texttt{DGC}}
\newcommand{\VaDE}{\texttt{VaDE}}
\newcommand{\FL}{\texttt{FL}}
\newcommand{\Conv}{\texttt{Conv}}
\newcommand{\ReLU}{\texttt{ReLU}}
\newcommand{\sig}{\texttt{Sigmoid}}
\newcommand{\maxp}{\texttt{MaxPool}}
\newcommand{\batch}{\texttt{BatchNorm2d}}
\usepackage[title]{appendix}
\usepackage[disable]{todonotes} 
\usepackage{color}

\title{Deep Goal-Oriented Clustering}

\author{%
  Yifeng Shi \ \ \ \ \ \ Christopher M. Bender  \ \ \ \ \ \ Junier B. Oliva   \ \ \ \ \ \ Marc Niethammer\\
  Department of Computer Science\\
  The University of North Carolina at Chapel Hill\\
  \texttt{\{yifengs, bender, joliva, mn\}@cs.unc.edu} \\
}

\begin{document}

\maketitle

\begin{abstract}
  Clustering and prediction are two primary tasks in the fields of unsupervised and supervised learning, respectively. Although much of the recent advances in machine learning have been centered around those two tasks, the interdependent, mutually beneficial relationship between them is rarely explored. One could reasonably expect appropriately clustering the data would aid the downstream prediction task and, conversely, a better prediction performance for the downstream task could potentially inform a more appropriate clustering strategy. In this work, we focus on the latter part of this mutually beneficial relationship. To this end, we introduce Deep Goal-Oriented Clustering ({\DGC}), a probabilistic framework that clusters the data by jointly using supervision via \textit{side-information} and unsupervised modeling of the inherent data structure in an end-to-end fashion. We show the effectiveness of our model on a range of datasets by achieving prediction accuracies comparable to the state-of-the-art, while, more importantly in our setting, simultaneously learning congruent clustering strategies.
\end{abstract}

\section{Introduction}

Much of the advances in supervised learning in the past decade can be credited to the development of deep neural networks (DNN), a class of hierarchical function approximators that are capable of learning complex input-output relationships. Prime examples, for the success of these approaches are, for example, image recognition~\cite{Krizhevsky2012}, speech recognition~\cite{Nassif2019}, and neural translation~\cite{Bahdanau2015}. However, with the explosion of the size of modern datasets, it becomes increasingly unrealistic to manually label all available data for training. Hence, the ability to understand the inherent data structure through unsupervised clustering in the absence of response values is of increasing importance.

Several attempts to apply DNNs to unsupervised clustering have been made in the past few years~\cite{Caron2018,Law2017,Xie2016,Shaham2018}, centering around the concept that the input space in which traditional clustering algorithms operate is of importance. Hence, learning this space from data is desirable, in particular, for complex data. Despite the improvements these approaches have made on benchmark clustering datasets, the ill-defined, ambiguous nature\todo{I still don't know what is ill-defined and ambiguous about it. It's what you define it to be, no? YS---It's ill-defined and ambiguous because for instance different distance metrics can produce very different results, and without prior knowledge, it's impossible to know which is better or what better even means.} of clustering still remains a challenge. Such ambiguity is particularly problematic in scientific discovery, sometimes requiring researchers to choose from different, but potentially equally meaningful clustering results when little information is available a priori~\cite{Ronan2016}. 

When facing such ambiguity, using as much available side-information as one can to obtain the most appropriate clustering is a fruitful direction to resolve clustering ambivalence~\cite{Xing2002,Khashabi2015,Jin2013}. Such side-information is usually available in terms of constraints, such as \textit{must-link} and \textit{cannot-link} constraints~\cite{Wang2010,Wagstaff2000}, or via some pre-conceived notion of similarity~\cite{Xing2002}. Defining such constraints requires human expertise, and is potentially vulnerable to noise or labeling errors. Such expert definitions resemble the aforementioned difficulty with supervised learning requiring a large amount of manually labeled training data; such an approach also precludes the possibility of learning from a considerable amount of \emph{indirect}, but informative side-information (which might not be easily quantifiable by an expert), while forming a congruous clustering strategy at the same time. 

\paragraph{Main Contributions} 
We propose \emph{Deep Goal-Oriented Clustering} ({\DGC}), a probabilistic model formalizing the notion of using informative side-information to form a pertinent clustering strategy. Our main contributions are: 1) We introduce a probabilistic model that combines supervision via side-information and unsupervised modeling of data structure when searching for a clustering strategy; 2) We make minimal assumptions on what form the supervised side-information might take, and assume no explicit correspondence between the side-information and the clusters; 3) We train \DGC\ end-to-end so that the model learns from the available side-information while forming a clustering strategy at the same time, thereby taking advantage of the mutually beneficial relationship between learning from informative side-information and forming a congruous clustering strategy.

\section{Related Work}

Most related work in the existing literature can be classified into two categories: 1) One class of methods can utilize extra side-information to form better, less ambiguous clusters, but such side-information needs to be provided beforehand and cannot be learned; 2) The other class of methods can learn from the provided labels to lessen the ambiguity in the formed clusters, but these methods rely on the \textit{cluster assumption} (detailed below), and usually assume that the provided labels are discrete and the \textit{ground truth labels}. This excludes the possibility of learning from indirectly related but informative side-information. We propose a unified framework that allows using informative side-information directly or indirectly to arrive at better formed clusters. 

\paragraph{Side-information as constraints} Using side-information to form better clusters is well-studied. Wagstaff et al.~\cite{Wagstaff2000} consider both the must-link and the cannot-link constraints in the context of K-means clustering. Motivated by image segmentation, Orbanz et al.~\cite{Orbanz2007} proposed a probabilistic model that can incorporate must-link constraints. Khashabi et al.~\cite{Khashabi2015} proposed a nonparametric Bayesian hierarchical model to incorporate noisy side-information as soft-constraints. In supervised clustering~\cite{Finley2005}, the side-information is the a priori known complete clustering for the training set, which is being used as a constraint to learn a mapping between the data and the given clustering. In contrast, in this work, we do not assume that constraints are given a priori. Instead, we let the side-information guide the clustering during the training process.
 
\paragraph{Semi-supervised methods \& The \textit{cluster assumption}} Semi-supervised clustering approaches generally assume that they only have access to a fraction of the true cluster labels. Via constraints as the ones discussed above, the available labels can then be propagated to unlabeled data. This strategy can help mitigate the ambiguity in choosing among different clustering strategies~\cite{Bair2013}. The generative approach to semi-supervised learning introduced in~\cite{Kingma2014} is based on a hierarchical generative model with two variational layers. Although it was originally meant for semi-supervised classification tasks, it can also be used for clustering. However, if used for clustering, it has to strictly rely on the so-called \textit{cluster assumption},
which states that there exists a direct correspondence between labels/classes and clusters~\cite{Frber2010,Chapelle2006}. Sansone et al.~\cite{Sansone2016} proposed a method for joint classification and clustering to address the sometimes stringent cluster assumption most approaches in the literature make by modeling the cluster indices and the class labels separately, underscoring the possibility that each cluster may consist of multiple class labels. Deploying a mixture of factor analysers as the underlying probabilistic framework, they also used a variational approximation to maximize the joint log-likelihood.

In this work, we generalize the notion of learning from discrete, ground truth labels to learning from indirect, but informative side-information. We make virtually no assumptions on the form of $\textbf{y}$ nor its relations to the clusters. This makes our approach potentially more applicable to general settings. 


\section{Background \& Problem Setup}
\subsection{Background---Variational Deep Embedding}

The starting point for \DGC\ is the \textit{variational auto-encoder} ({\VAE})~\cite{Kingma2014AutoEncoding} with a prior distribution of the latent code chosen as a Gaussian mixture distribution and introduced in~\cite{Jiang2017}, as {\VaDE}. We briefly review the generative \VaDE\ approach here to provide the background for {\DGC}. We adopt the notation that lower case letters denote random variables; bold, lower case letters denote random vectors or random samples from them; and bold upper case letters denote random matrices. 

Assume the prior distribution of the latent code, $\textbf{z}$, belongs to the family of Gaussian mixture distributions, i.e. $p(\textbf{z}) = \sum_c p(\textbf{z}|c)p(c) = \sum_c \pi_c \mathcal{N}(\mu_c,\sigma^2_c\textbf{I})$ where $c$ is a random variable, with prior probability $\pi_c$, indexing the normal component with mean $\mu_c$ and variance $\sigma_c^2$. \VaDE\ allows for the clustering of the input data in the latent space, with each component of the Gaussian mixture prior representing an underlying cluster. A {\VAE}-based model can be efficiently described in terms of its generative process and inference procedure. We start with the generative process of {\VaDE}. Given an input $\textbf{x}\in R^d$, the following decomposition of the joint probability $p(\textbf{x},\textbf{z},c)$ details its generative process: $p(\textbf{x},\textbf{z},c) = p(\textbf{x}|\textbf{z})p(\textbf{z}|c)p(c)$. In words, we first sample the component index $c$ from a prior categorical distribution $p(c)$, then sample the latent code \textbf{z} from the component $p(\textbf{z}|c)$, and lastly reconstruct the input $\textbf{x}$ through the reconstruction network $p(\textbf{x}|\textbf{z})$. To perform inference and learn from the data, just as for regular {\VAEs}, a \VaDE\ is constructed to maximize the log-likelihood of the input data $\textbf{x}$ by maximizing its \textit{evidence lower bound} (ELBO):
\begin{equation} \label{vade}
    \begin{split}
    \log p(\textbf{x}) &\geq  \mathbb{E}_{q(\textbf{z}|\textbf{x})}\log p(\textbf{x}|\textbf{z}) - \mathbb{E}_{q(c|\textbf{x})}\log \frac{q(c|\textbf{x})}{p(c)} - \mathbb{E}_{q(\textbf{z}, c|\textbf{x})} \log \frac{q(\textbf{z}|\textbf{x})}{p(\textbf{z}|c)}\enspace,
    \end{split}
\end{equation}
where $q(\textbf{z},c|\textbf{x})$ denotes the variational posterior distribution over the latent variables given the input $\textbf{x}$. With proper assumptions on the prior and variational posterior distributions, the ELBO in Eq.~\eqref{vade} admits a closed-form expression in terms of the parameters of those distributions.
We refer readers to \cite{Jiang2017} for more technical details.

\subsection{Problem Setup}

Unlike the previously described unsupervised setting, we \emph{do} assume we have a response variable $\textbf{y}$, and our goal is to use \textbf{y} to inform a better clustering strategy. Abstractly, given the input-output random variable pair $(\textbf{x}, \textbf{y})$, we seek to divide the probability space of $\textbf{x}$ into non-overlapping subspaces that are meaningful in explaining the output \textbf{y}. In other words, we want to use the prediction task of mapping data points, $x$, sampled from the probability space of \textbf{x} to their corresponding sampled outcomes $y$ as a \textit{teaching agent}, to guide the process of dividing the probability space of $\textbf{x}$ into subspaces that optimally explain $y$. Since our goal is to discover the aforementioned subspace-structure without having a priori knowledge whether such a structure indeed exists, a probabilistic framework is more appropriate because of its ability to incorporate and reason with uncertainty. 

To this end, we use and extend the \VaDE\ framework, with the following assumption imposed on the latent code that specifically caters to our setting: Assume the input $\textbf{x}$ carries predictive information with respect to the output $\textbf{y}$. Since the latent code $\textbf{z}$ should inherit sufficient information from which the input $\textbf{x}$ can be reconstructed, it is reasonable to assume that $\textbf{z}$ also inherits that predictive information. This assumption allows us to assume that $\textbf{x}$ and $\textbf{y}$ are conditionally independent given $\textbf{z}$, i.e. $p(\textbf{x}, \textbf{y}|\textbf{z}) = p(\textbf{x}|\textbf{z}) p(\textbf{y}|\textbf{z})$.

\section{Deep Goal-Oriented Clustering}
\subsection{Generative Process}

\begin{wrapfigure}[15]{r}{0.4\textwidth}
\centering
\vspace{-26mm}
\includegraphics[width=1.0\linewidth]{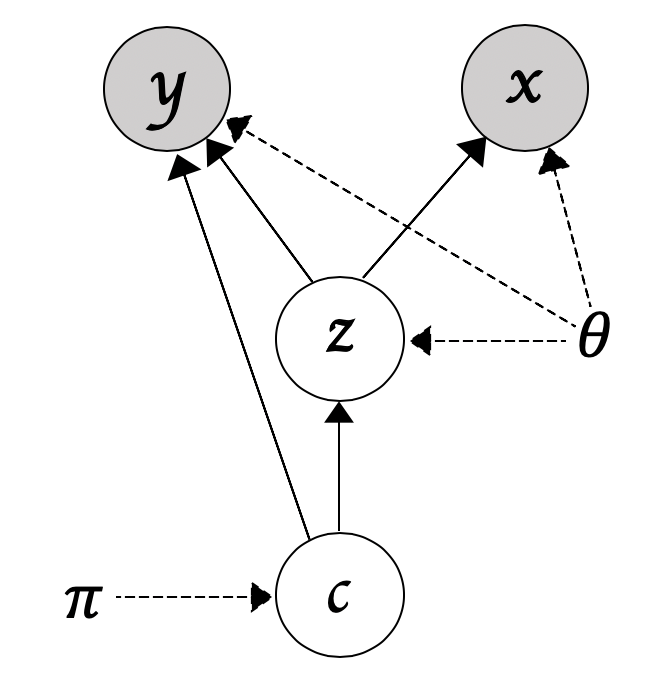}
\caption{The Bayesian network that underlies the generative process of {\DGC}. $\theta$ and $\pi$ together constitute the generative model parameters.}
\label{fig:bayesian_net}
\end{wrapfigure}

We inherit the standard distributional assumptions on the auto-encoding (unsupervised) part of our model from \cite{Jiang2017}. As for incorporating $\textbf{y}$ into this probabilistic model, recall from our previous discussion that $\textbf{y}$ might manifest with respect to the input differently across different subspaces of the input space if, indeed, such subspace-structure exists. Viewing $p(\textbf{y}|\textbf{z})$ as a conditional probability distribution over $\textbf{y}$ resulting from a functional transformation from \textbf{z} to the space of probability distributions over \textbf{y},  we can assume that the ground truth transformation function, $g_c$, is different for each subspace indexed by $c$. More concretely, if $\textbf{z}\sim p(\textbf{z}|c)$ for some index $c$, we assume $p(\textbf{y}|\textbf{z},c) \propto g_c(\textbf{z})$ for some subspace-specific $g_c$. As a result, we learn a different mapping function for each subspace.

The overall generative process of our model is as follows: \textbf{1}. Generate $c\sim \textrm{Cat}(\pi)$; \textbf{2}. Generate $\textbf{z} \sim p(\textbf{z}|c)$; \textbf{3}. Generate $x \sim p(\textbf{x}|\textbf{z})$; \textbf{4}. Generate $y \sim p(\textbf{y}|\textbf{z},c)$. The Bayesian network that underlies \DGC\ is shown in Fig.~\ref{fig:bayesian_net}, and with the previously stated assumption, the joint distribution can thus be decomposed as: $p(\textbf{x},\textbf{y},\textbf{z},c) = p(\textbf{y}|\textbf{z},c) p(\textbf{x}|\textbf{z}) p(\textbf{z}|c) p(c)$.

\subsection{Inference \& Variational Lower Bound}

With the aforementioned setup and the assumption that the variational posterior distribution $q(\textbf{z},c|\textbf{x})$ belongs to the class of mean-field distributions \cite{Wainwright2008}, i.e. $q(\textbf{z},c|\textbf{x}) = q(\textbf{z}|\textbf{x}) q(c|\textbf{x})$, we have the following variational lower bound (see the Appendix for a detailed derivation)
	\begin{equation} \label{elbo}
        \begin{split}
        \log p(\textbf{x},\textbf{y}) 
        &\geq \underbrace{\mathbb{E}_{q(\textbf{z},c|\textbf{x})}\log p(\textbf{y}|\textbf{z},c)}_{\text{Probabilistic Ensemble}} + \underbrace{\mathbb{E}_{q(\textbf{z},c|\textbf{x})}\log \frac{p(\textbf{x},\textbf{z},c)}{q(\textbf{z},c|\textbf{x})}}_\text{ELBO for \VAE\ with GMM prior} = \mathcal{L}_{\text{ELBO}}\,.
        \end{split}
    \end{equation}
We point out an implication of the first term in $\mathcal{L}_{\text{ELBO}}$: in essence, it allows for a probabilistic ensemble of classifiers based on the subspace index. This can be seen as follows
$$\mathbb{E}_{q(\textbf{z},c|\textbf{x})}\log p(\textbf{y}|\textbf{z},c) = \mathbb{E}_{q(\textbf{z}|\textbf{x})}\left[\sum_k \lambda_k\log p(\textbf{y}|\textbf{z},c=k)\right] \approx \frac{1}{M}\sum_{l=1}^M \left[\sum_k \lambda_k\log p(\textbf{y}|\textbf{z}^{(l)},c=k)\right]$$
where $\lambda_k = q(c=k|\textbf{x})$ and $l$ indexes the Monte Carlo samples used to approximate the expectation with respect to $q(\textbf{z}|\textbf{x})$. This is sensible, especially given that our goal is to discover appropriate subspace structure in the latent space without knowing whether one exists a priori. The probabilistic ensemble in this scenario offers the model the luxury to maintain necessary uncertainty with respect to the discovered subspace structure until an unambiguous structure is captured. We next discuss how to appropriately choose the variational posterior distributions.


\subsection{Mean-field Variational Posterior Distributions}
Following the variational neural net approaches \cite{Kingma2014}, we choose $q(\textbf{z}|\textbf{x})$ to be $\mathcal{N}\left(\textbf{z}|\pmb{\Tilde{\mu}}_{\textbf{z}}, \pmb{\Tilde{\sigma}}^2_{\textbf{z}}\textbf{I}\right)$ where $\left[\pmb{\mu}_{\textbf{z}}, \pmb{\sigma}^2_{\textbf{z}}\right] = h(\textbf{x};\theta)$, with $h$ being parametrized by a feed-forward neural network with weights $\theta$. Although it may seems unnatural to use a unimodal distribution to approximate a multimodal distribution, when the learned $q(c|\textbf{x})$ becomes discriminative, dissecting the original $\mathcal{L}_{\text{ELBO}}$ derived in Eq.~\eqref{elbo} in the following way indicates that such an approximation will not incur a sizeable information loss (see Appendix for a detailed derivation):
\begin{equation} \label{diff_decomp}
\begin{aligned}
    \mathcal{L}_{\text{ELBO}} = \mathbb{E}_{q(\textbf{z},c|\textbf{x})}\log p(\textbf{y}|\textbf{z},c) + & \mathbb{E}_{q(\textbf{z}|\textbf{x})}\log p(\textbf{x}|\textbf{z}) \\
    & - \mathbb{KL}\left(q(c|\textbf{x})||p(c)\right) - \sum_{k}\lambda_k\mathbb{KL}\left(q(\textbf{z}|\textbf{x})||p(\textbf{z}|c=k)\right)\,,
\end{aligned}
\end{equation}
where $\lambda_k$ denotes $q(c=k|\textbf{x})$. Analyzing the last term in Eq.~\eqref{diff_decomp}, we notice that if the learned variational posterior $q(c|\textbf{x})$ is very discriminative and puts most of its weight on one specific index $c$, all but one $\mathbb{KL}$ terms in the weighted sum will be close to zero. Therefore, choosing $q(\textbf{z}|\textbf{x})$ to be unimodal to minimize that specific $\mathbb{KL}$ term would be appropriate, as $p(\textbf{z}|c)$ is assumed to be a unimodal normal distribution for all $c$.

Choosing $q(c|\textbf{x})$ appropriately requires us to analyze the proposed $\mathcal{L}_{\text{ELBO}}$ in greater detail based on the following decomposition (see the Appendix for a detailed derivation):
	\begin{equation} \label{decomp}
        \begin{split}
        \mathcal{L}_{\text{ELBO}} &= \underbrace{ \mathbb{E}_{q(\textbf{z},c|\textbf{x})}\log p(\textbf{y}|\textbf{z},c)}_{\circled{1}} + \underbrace{\mathbb{E}_{q(\textbf{z}|\textbf{x})}\log \frac{p(\textbf{x},\textbf{z})}{q(\textbf{z}|\textbf{x)}}}_{\circled{2}} - \underbrace{ \mathbb{E}_{q(\textbf{z}|\textbf{x})} \mathbb{KL}\left(q(c|\textbf{x})||p(c|\textbf{z})\right)}_{\circled{3}}.
        \end{split}
    \end{equation}
        We observe that since $\circled{2}$ does not depend on $c$, $q(c|\textbf{x})$ should be chosen to maximize $\circled{1} - \circled{3}$. Moreover, the expectation over $q(\textbf{z}|\textbf{x})$ does not depend on $c$, and thus has no influence over our choice of $q(c|\textbf{x})$. Casting finding $q(c|\textbf{x})$ as an optimization problem, we have
\begin{equation} \label{opti}
\begin{aligned}
\min_{q(c|\textbf{x})} \quad & f_0(q) =  \mathbb{KL}\left(q(c|\textbf{x})||p(c|\textbf{z})\right) -  \mathbb{E}_{q(c|\textbf{x})} \log p(\textbf{y}|\textbf{z},c)\,,\\
\textrm{s.t.} \quad & \sum_k q(c=k|\textbf{x}) = 1,\quad
  q(c=k|\textbf{x}) \geq 0,\ \ \forall k\,.
\end{aligned}
\end{equation}
The objective functional $f_0$ is convex over the probability space of $q$, as the \textit{Kullback–Leibler divergence} is convex in $q$ and the expectation is linear in $q$. Analytically solving the convex program (\ref{opti}) (see the Appendix for a detailed derivation), we obtain
\begin{equation} \label{chooseq}
  q(c=k|\textbf{x})  = \frac{p(\textbf{y}|\textbf{z},c=k)\cdot p(c=k|\textbf{z})}{\sum_k p(\textbf{y}|\textbf{z},c=k)\cdot p(c=k|\textbf{z})}\,. 
\end{equation}
 The form of $q$ derived in Eq.~\eqref{chooseq} aligns with our intuition. To better facilitate understanding, we interpret Eq.~\eqref{chooseq} in two extremes. If $\textbf{y}$ is evenly distributed across the different subspaces, i.e. the ground truth transformations $g_c$ are the same, then $q(c=k|\textbf{x}) = p(c=k|\textbf{z})$, which is what one would choose for unsupervised clustering~\cite{Jiang2017}. However, if the supervised task is informative while the unsupervised task is not, i.e. $p(c|\textbf{z})$ is a uniform distribution, the likelihoods $\{p(\textbf{y}|\textbf{z},c=k)\}_k$ would dominate $q$. Therefore, one could interpret any in-between scenario as a balance that automatically weights the supervised and the unsupervised tasks based on how strong their signals are with respect to grouping the latent probability space into different subspaces. 

\subsection{Evaluating on Unlabeled Data}
When presented with the response variable $\textbf{y}$, Eq.~\eqref{chooseq} gives the optimal choice of $q(c|\textbf{x})$ that allows the network to incorporate both the supervised and unsupervised signals when weighting the clusters. Nevertheless, just as with most supervised tasks, one usually does not have access to \textbf{y} on newly collected (test) data points, which prohibits us from evaluating $q(c|\textbf{x})$. One easy remedy to this would be to use $p(c|\textbf{z})$ when $\textbf{y}$ is not available; however, having an ensemble of well-trained conditional likelihood mappings, $\{p(\textbf{y}|\textbf{z},c=k)\}_k$, and not utilizing it when evaluating on new data points seems wasteful. We now show that by incorporating a regularization term to $\mathcal{L}_{\text{ELBO}}$, \DGC\ can be naturally generalized to unlabeled testing samples. We define our regularized ELBO as:
\begin{equation} \label{regu}
    \begin{aligned}
        \mathcal{L}^{\textbf{regu}}_{\text{ELBO}} =\mathcal{L}_{\text{ELBO}} - \mathbb{E}_{q(\textbf{z},c|\textbf{x})}\mathbb{H}_{\textbf{max}}\left(p(\textbf{y}|\textbf{z},c)\right)\,,
    \end{aligned}
\end{equation}
where $\mathbb{H}_{\textbf{max}}\left(p(\textbf{y}|\textbf{z},c)\right)= \mathbf{max}\{\mathbb{H}\left(p(\textbf{y}|\textbf{z},c)\right),0\}$ and $\mathbb{H}\left(p(\textbf{y}|\textbf{z},c)\right) = -\mathbb{E}_{p(\textbf{y}|\textbf{z},c)}\log p(\textbf{y}|\textbf{z},c)$. We make a few additional comments here. Firstly, on the one hand if $\textbf{y}$ is a discrete random variable, $\mathbb{H}_{\textbf{max}}\left(p(\textbf{y}|\textbf{z},c)\right)= \mathbb{H}\left(p(\textbf{y}|\textbf{z},c)\right)$, which is the entropy of $p(\textbf{y}|\textbf{z},c)$ which is always non-negative; on the other hand when $\textbf{y}$ is continuous, although the differential entropy of $p(\textbf{y}|\textbf{z},c)$ can take any sign, the $\mathbf{max}$ operator ensures that $\mathbb{H}_{\textbf{max}}\left(p(\textbf{y}|\textbf{z},c)\right)$ will remain non-negative. Therefore, adding (a convex combination of) negative entropies preserves the inequality, and thus $\mathcal{L}^{\textbf{regu}}_{\text{ELBO}}$ remains a proper lower bound.
Secondly, just as we obtained $q(c|\textbf{x})$ in Eq.~\eqref{chooseq}, solving a similar convex program still provides the optimal choice of $q(c|\textbf{x})$ in the presence of the regularizer
\begin{equation} \label{chooseq_regu}
  q(c=k|\textbf{x})  = \frac{e^{\log p(\textbf{y}|\textbf{z},k)-\mathbb{H}_{\textbf{max}}\left(p(\textbf{y}|\textbf{z},c)\right)}\cdot p(k|\textbf{z})}{\sum_j e^{\log p(\textbf{y}|\textbf{z},j)-\mathbb{H}_{\textbf{max}}\left(p(\textbf{y}|\textbf{z},c)\right)}\cdot p(j|\textbf{z})} \,.
\end{equation}
Finally, it is crucial to note that this form of regularization penalizes the conditional distributions, $p(\textbf{y}|\textbf{z},c)$. Hence, clusters with higher posterior weights are the ones that have low entropies. This aligns with our intuition: the most suitable cluster to explain a given $\textbf{y}$ should be more certain in how it is distributed. This attribute, once appropriately cultivated, can be used as a replacement for the likelihoods in Eq.~\eqref{chooseq}, as entropy characterizes the distribution, but does not require concrete samples $\textbf{y}$ to be calculated. More specifically, when evaluating on an unlabeled data point, we use
\begin{equation} \label{chooseq_test}
  q^{\text{test}}(c=k|\textbf{x})  = \frac{e^{-\mathbb{H}_{\textbf{max}}\left(p(\textbf{y}|\textbf{z},c)\right)}\cdot p(k|\textbf{z})}{\sum_j e^{ -\mathbb{H}_{\textbf{max}}\left(p(\textbf{y}|\textbf{z},c)\right)}\cdot p(j|\textbf{z})} 
\end{equation}
to weight the clusters, aligning with our previous reasoning; the cluster that corresponds to $p(\textbf{y}|\textbf{z},c)$ with the lowest entropy will be weighted most heavily. This allows the model to use the tuned conditional likelihood mappings even when $\textbf{y}$ is not available.


\section{Experiments}

We investigate the efficacy of our proposed \DGC\ on a range of datasets. We refer the reader to the Appendix for experimental details. E.g., regarding the train/validation/test split, the chosen network architecture, the choices of learning rate and optimizer.




\subsection{Noisy MNIST}


We introduce a synthetic data experiment using the MNIST dataset, which we name the \textit{noisy MNIST}, to illustrate that the supervised part of \DGC\ can enhance the performance of an otherwise well-performing unsupervised counterpart. Further, we explore the behavior of \DGC\ without its unsupervised part to demonstrate the importance of capturing the inherent data structure. The dataset is created as follows. We extract images that correspond to the digits 2 and 7 from MNIST. For each digit, we randomly select half of the images for that digit and superpose noisy backgrounds onto those images, where the backgrounds are cropped from randomly selected CIFAR-10 images. See Fig.~\ref{fig:true} for an example, where the two images in the first row are normal MNIST digits, and the images in the second row are the ones superposed with noisy backgrounds. Each image is also associated with a binary response variable $\textbf{y}$, indicating if the digit is a 2 or a 7. Our goal is to cluster the images into 4 clusters: digits 2 and 7, with and without background. However, we are only using the binary responses for supervision and have no direct knowledge of the background.


\begin{figure*}[t]
\centering
\begin{tabular}{ccc}
\begin{tabular}{cc}
\begin{subfigure}{.175\textwidth}
  \centering
  \includegraphics[width=1.0\linewidth]{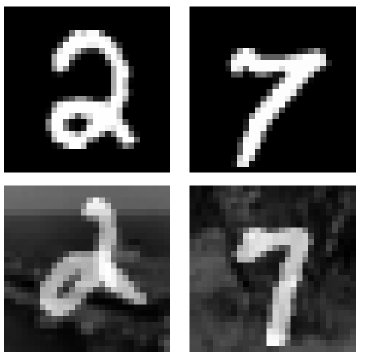}
  \caption{Real}
  \label{fig:true}
\end{subfigure}
\\
\begin{subfigure}{.175\textwidth}
  \centering
  \includegraphics[width=1.0\linewidth]{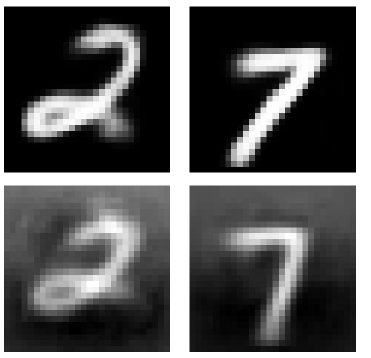}
  \caption{Generated}
  \label{fig:gen}
\end{subfigure}
\end{tabular}
&
\begin{subfigure}{.35\textwidth}
  \centering
  \includegraphics[width=1.0\linewidth]{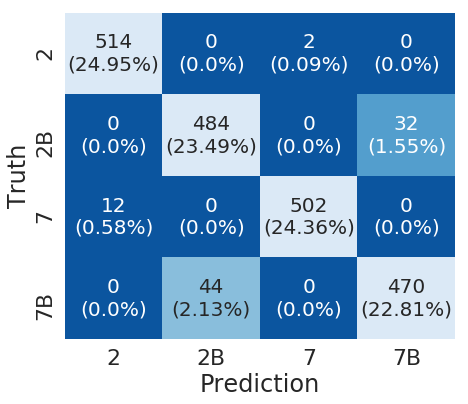}
  \caption{Confusion Matrix---\VaDE}
  \label{fig:conf_vae}
\end{subfigure}
&
\begin{subfigure}{.35\textwidth}
  \centering
  \includegraphics[width=1.0\linewidth]{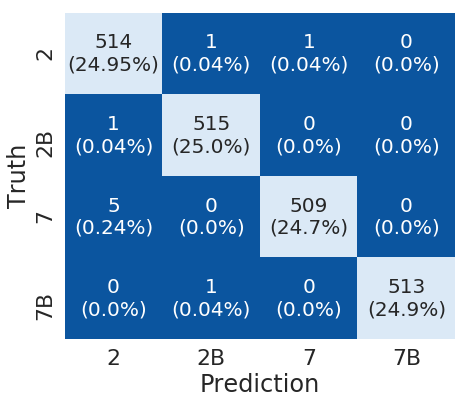}
  \caption{Confusion Matrix---\DGC}
  \label{fig:conf_dgc}
\end{subfigure}
\end{tabular}
\caption{(a) Created ground truth images; (b) Samples from a trained {\DGC}. For the confusion matrices presented in (c) and (d), the abbreviation, 2B/7B, in the row/column labels denotes digits 2/7 with background. For both (c) and (d), rows represent the predicted clustering labels, and columns represent the ground truth clustering labels. E.g., if the entry in the matrix that corresponds to the predicted label, ``2B'', and the true label, ``7B'', is 44 (as it is the case in (c)), this means {\VaDE} misclassified 44 images. Specifically, true 7's with background were confused with 2's with background in this case.}
\label{fig:mnist}
\end{figure*}

As a baseline, \VaDE\ already performs well on this dataset, achieving a clustering accuracy of 95.6\% when the desired number of clusters is set to 4. 
Fig.~\ref{fig:conf_vae} shows that \VaDE\ distinguishes well based on the presence or absence of the noisy background. Incorrectly clustered samples are mainly due to {\VaDE}'s inability to differentiate the underlying digits. This is reasonable behavior: if the background signal dominates, the network may focus on the background for clustering as it has no explicit knowledge about the digits.

{\DGC}, on the other hand, performs nearly perfectly (with a clustering accuracy of 99.6\%) in this setting with the help of the added supervision. Firstly, we see that \DGC\ mitigates the difficulty of distinguishing between digits under the presence of strong, noisy backgrounds (as shown in Fig.~\ref{fig:conf_dgc}, where \DGC\ makes almost no mistakes in distinguishing between digits even in the presence of noisy backgrounds). Secondly, this added supervision does not overshadow the original advantage of \VaDE\ (i.e. distinguishing whether the images contain background or not). Instead, it enhances the overall model on cases where the unsupervised part, i.e., {\VaDE}, struggles. Lastly, as detailed in~\cite{Sansone2016} and earlier sections, most existing approaches that take advantage of available labels rely on \textit{the cluster assumption}, which assumes a direct relationship between the clusters and the labels used for supervision. This experiment is a concrete example that demonstrates that \DGC\ does not need to rely on such an assumption to form a sound clustering strategy. Instead, \DGC\ is able to work with class labels that are only partially indicative of what the final clustering strategy should be, potentially making \DGC\ more applicable to more general settings.

\paragraph{Ablation study} To further test the importance of each part of our model, we ablate the probabilistic components (i.e. we get rid of the decoder and the loss terms associated with it, so that only the supervision will inform how the clusters are formed in the latent space) and perform clustering using only the supervised part of our model. We find that clustering accuracy degrades from the nearly-perfect accuracy obtained by the full model to 50\%. Coupled with the improvements over {\VaDE}, this indicates that each component of our model contributes to the final accuracy and that our original intuition that supervision and clustering may reinforce each other is correct.

\subsection{Pacman}

In this experiment we test {\DGC}'s ability to learn a clustering strategy when using a continuous response as \textit{side-information}. 
The Pacman-shaped data, depicted in Fig.~\ref{fig:1}, consists of two annuli. Each point in the two annuli is associated with a continuous response value. These response values decrease linearly (from 1 to 0) in one direction for the inner (yellow) annulus, and increase exponentially (from 0 to 1) in the opposite direction for the outer (purple) annulus. See Fig.~\ref{fig:2} for a graphical illustration of the response values and Fig.~\ref{fig:3} for an overall 3D illustration of the dataset. We choose the linear/exponential rates for the responses to not only test our model's ability to detect different trends, but also to test its ability to fit different functional trends.

\begin{figure}[h]
    \centering 
\begin{subfigure}{0.25\textwidth}
  \includegraphics[width=\linewidth]{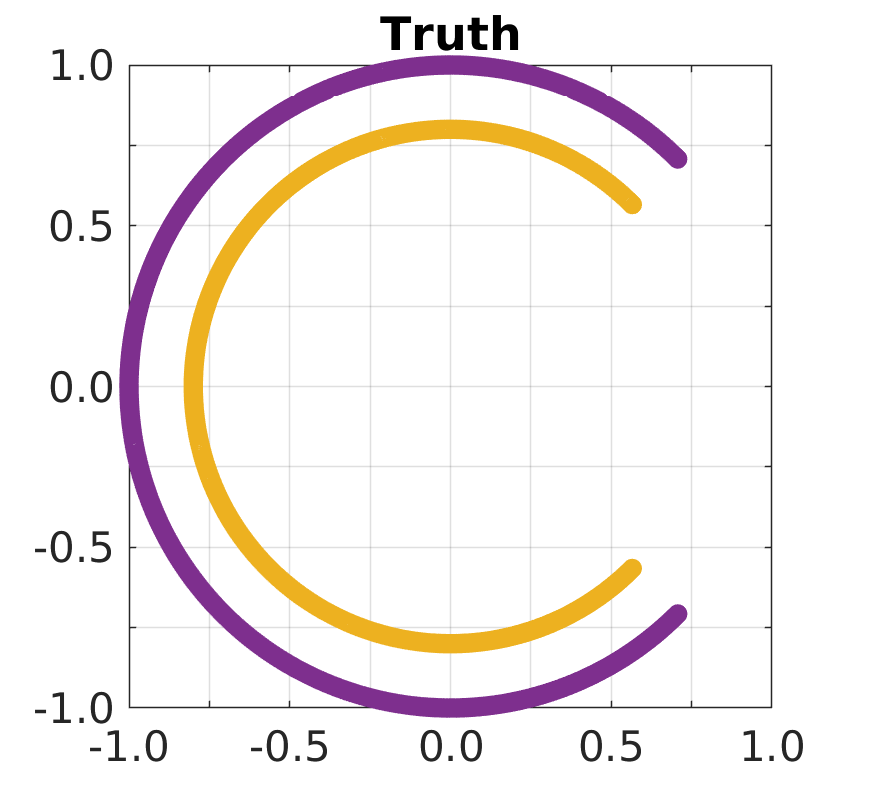}
  \caption{Pacman Inputs}
  \label{fig:1}
\end{subfigure}\hfil 
\begin{subfigure}{0.25\textwidth}
  \includegraphics[width=\linewidth]{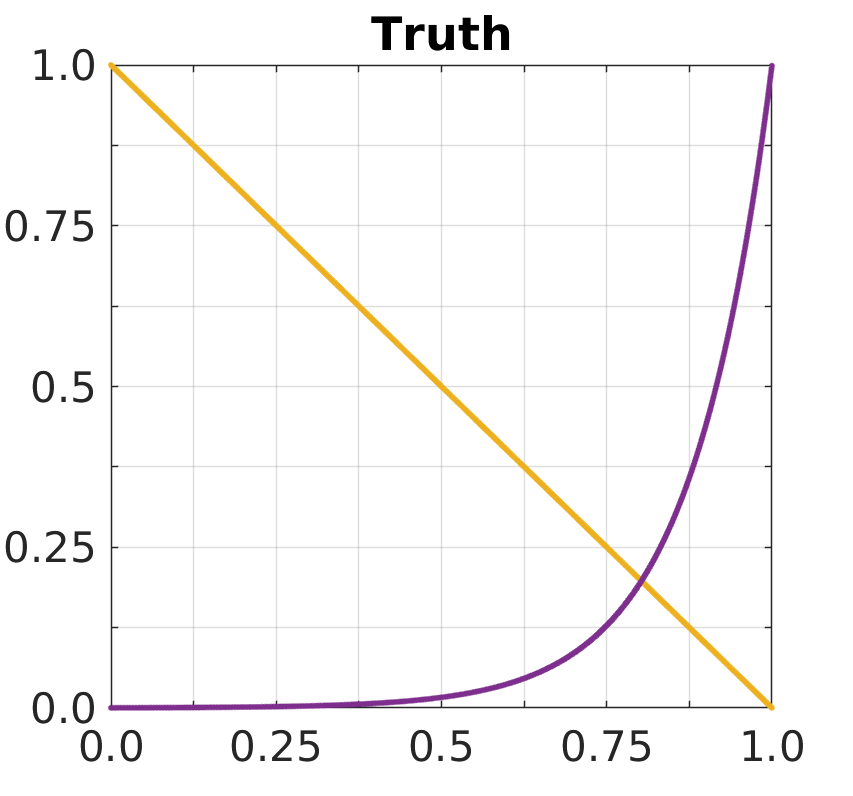}
  \caption{Pacman Task}
  \label{fig:2}
\end{subfigure}\hfil 
\begin{subfigure}{0.25\textwidth}
  \includegraphics[width=\linewidth]{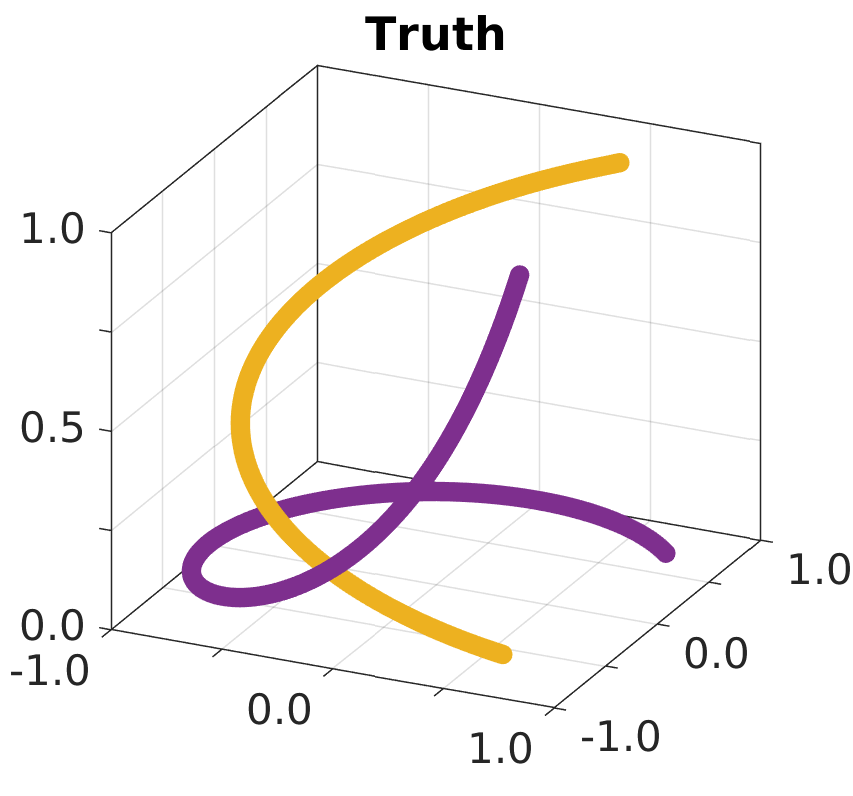}
  \caption{Pacman Task vs Inputs}
  \label{fig:3}
\end{subfigure}

\medskip
\begin{subfigure}{0.25\textwidth}
  \includegraphics[width=\linewidth]{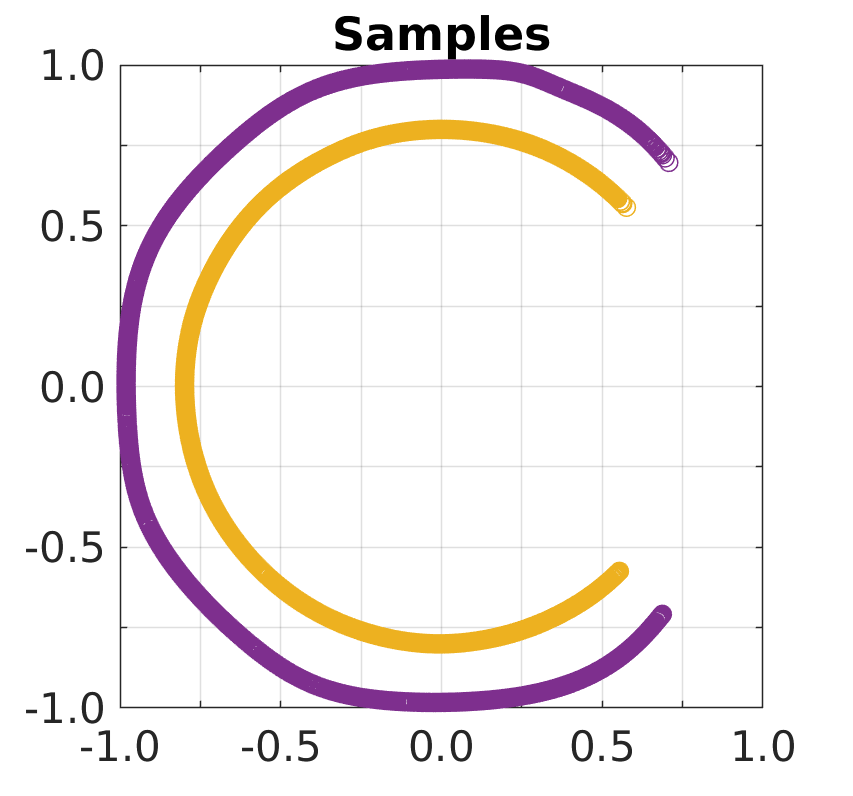}
  \caption{Pacman Inputs}
  \label{fig:4}
\end{subfigure}\hfil 
\begin{subfigure}{0.25\textwidth}
  \includegraphics[width=\linewidth]{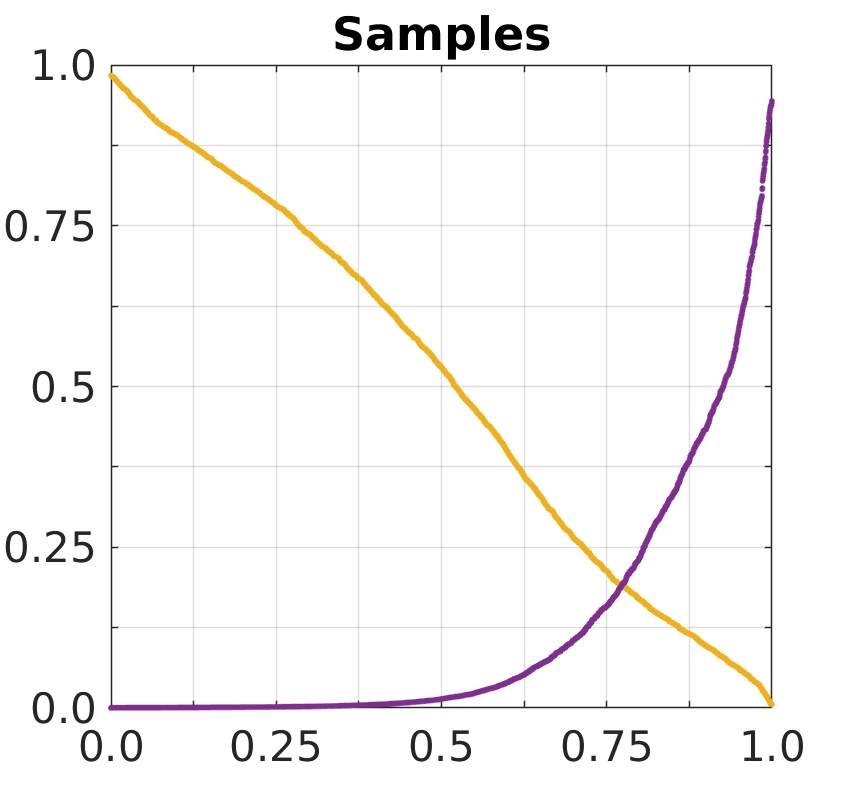}
  \caption{Pacman Task}
  \label{fig:5}
\end{subfigure}\hfil 
\begin{subfigure}{0.25\textwidth}
  \includegraphics[width=\linewidth]{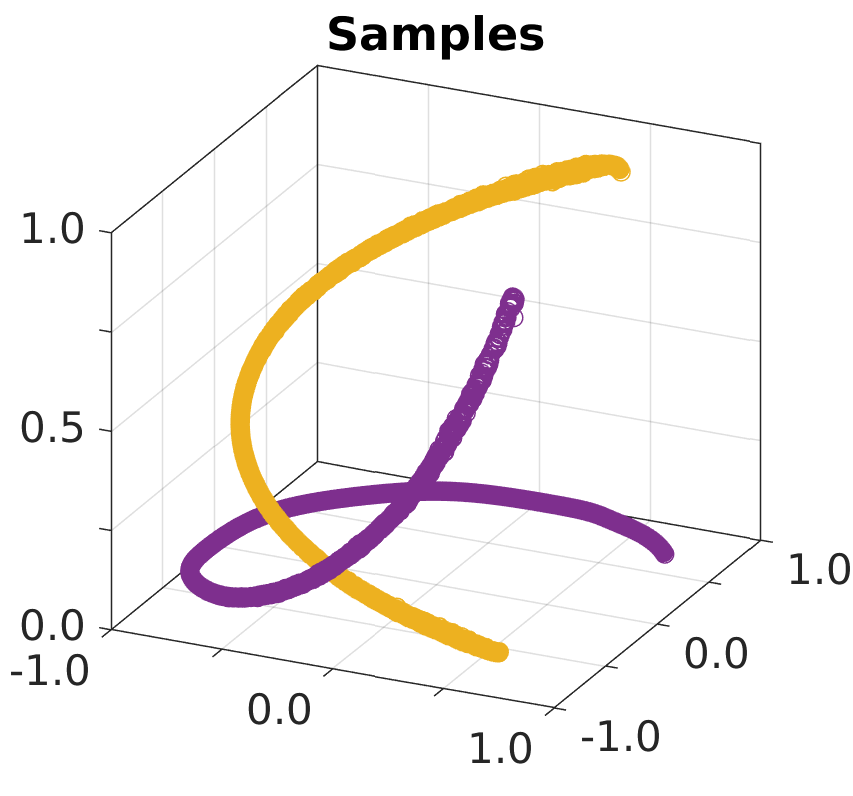}
  \caption{Pacman Task vs Inputs}
  \label{fig:6}
\end{subfigure}
\caption{The first row shows the ground truth 2D Pacman, the responses \textbf{y} alone, and the combined 3D Pacman. The second row depicts the corresponding generated counterparts from {\DGC}.}
\label{fig:pacman}
\end{figure}

Our goal is to separate the two annuli depicted in Fig.~\ref{fig:1}. This is challenging as the annuli were deliberately chosen to be very close to each other. First, we applied various unsupervised learning methods including K-means, hierarchical clustering, and {\VaDE}, to only the 2D Pacman-shaped data (i.e., not using the responses, but only the 2D Cartesian coordinates). None of the unsupervised methods managed to separate the two annuli. These approaches also tend to results in different clusterings as they are based on different distance metrics (see the Appendix for these results). This phenomenon echos a deep-rooted obstacle for clustering methods in general: the concept of clusters is inherently subjective, and different distance metrics can potentially produce different, but sometimes equally meaningful, clustering results.\todo{While this difficulty and subjectivity might be true it is not really strengthening your case, is it? It seems like you have exactly the same issues.---YS: I'm not sure I follow here. The point is the unsupervised methods have trouble distinguishing the two annuli based on arbitrarily chosen distance metric. With the added supervision, we can. --- i think MN just wants you to say this specifically, after you note the ambiguity as a way to juxtapose what we're doing vs what's normally done} 
Indeed, although visually obvious to humans, any distance metric would have difficulty recognizing the two annuli when they are sufficiently close to each other.

Applying \DGC\ with the input $\textbf{x}$ being the 2D Cartesian point coordinates and the responses $\textbf{y}$ being the response values described previously, we achieve training/testing clustering accuracies of 100\%
, distinguishing the two annulli wholly based on the discriminative information carried by the responses. Further, Fig.~\ref{fig:pacman} (second row) shows that the generated samples from \DGC\ can recover the structure of the original data. Both the Pacman shape and its corresponding response values are captured.

The following is worth noting. Firstly, the generated samples from \DGC\ substantiate the model's ability to appropriately learn and use the side-information provided by the response values to obtain a sensible clustering strategy. Secondly, unlike most previously discussed methods, \DGC\ can work with continuous response values. This is highly attractive, as it lends itself to any general regression setting in which one would believe the desired clustering should be informed by the regression task. 

\subsection{SVHN}

\begin{wraptable}[12]{r}{0.5\linewidth}
\vspace{-6.3mm}
  \begin{center}
    \caption{This table depicts how the testing classification and clustering accuracies vary with respect to the desired number of clusters.}
    \label{table:svhn_table}
    \small
    \begin{tabular}{l|c|c} 
      Number of Clusters & Classification & Clustering\\ 
      \hline
      1 Cluster & 92.7\% & 19.6\% \\
      \hline
      5 Clusters & 92.6\% & 57.8\% \\
      \hline
      \textbf{10 Clusters} & \textbf{94.7}\pmb{\%} & \textbf{92.6}\pmb{\%} \\
      \hline
      15 Clusters & 92.0\%  & 92.0\% \\
      \hline
      20 Clusters & 92.2\% & 91.9\% \\
      \hline
      50 Clusters & 91.3\%  & 87.9\% \\
      \hline
      100 Clusters & 89.3\%  & 84.1\% \\
    \end{tabular}
  \end{center}
\end{wraptable}
We apply \DGC\ to the Street View House Number (SVHN) dataset \cite{Netzer2011}. This dataset consists of 73,257 training images, 26,032 test images, and 531,131 additional training images (that are easier than the ones in the training set). We follow the standard procedure to pre-process the images \cite{Zagoruyko2016,Devries2017}, only normalizing them so that the pixel values are within the [0,1] range. We train our model using both the training images and the extra images, and do not use any data augmentation techniques. For the purpose of clustering, we use the class labels (of the digits) as the ground truth clustering labels to calculate the clustering accuracy. Applying \DGC\ on SVHN with the input $\textbf{x}$ being the street view images and the output $\textbf{y}$ being the class labels of the underlying SVHN digits, we achieve a final classification accuracy of 94.7\% and clustering accuracy of 92.6\%. 

Results show the following. Firstly, although we use the class labels as the responses to help the clustering, we only assume the \emph{number} of clusters. This is to say, \DGC\ has to learn the provided labels from scratch, and decide for itself how to group the images in the latent space based on the supervised and the unsupervised signals. With the help of the supervision, \DGC\ achieves, as expected, much higher clustering accuracy than the state-of-the-art unsupervised clustering approaches on SVHN (76.8\%)~\cite{Kilinc2018}. Secondly, the desired number of clusters is a hyperparameter in {\DGC}. We investigate how sensitive \DGC\ is with respect to this hyperparameter. Table \ref{table:svhn_table} shows that the classification accuracy is the highest when the desired number of clusters is 10. This aligns with our expectation as we know a priori that there are 10 underlying digits in this dataset. 

When the desired number of clusters is different from 10, we can still calculate the clustering accuracy with the caveat that the upper bound for the accuracy may not be 100\%. For instance, if we set the desired number of clusters to 1, then the clustering accuracy is bounded by the percentage of the largest cluster, whereas, if we set the desired number of clusters to be larger than the number of digits, the network can still achieve a clustering accuracy of 100\% if it learns to consistently group samples into a consistent set of clusters, with the cardinality of the set matching the number of digits. 
This is echoed by the clustering accuracies in Table~\ref{table:svhn_table}. When the desired number of clusters exceeds 10, \DGC\ is still able to achieve high clustering accuracies, hinting at the network's ability, when given enough flexibility, to learn to choose an appropriate number of clusters.

\section{Conclusion}

\todo{This conclusion needs to be strengthened. In particular, infuse it with a bit more excitement and discuss some exciting avenues for future research.}
In this work, we introduced {\DGC}, a probabilistic framework that allows for the integration of both supervised and unsupervised information when searching for a congruous clustering in the latent space.
This is an extremely relevant, but daunting task, where previous attempts are either largely restricted to discrete, supervised, ground-truth labels or rely heavily on the side-information being provided as manually tuned constraints.
To the best of our knowledge, this is the first attempt to simultaneously learn from generally indirect, but informative side-information and form a sensible clustering strategy, all the while making minimal assumptions on either the form of the supervision or the relationship between the supervision and the clusters.
This method is  applicable to a variety of fields where an instance's input and task are defined but its membership is important and unknown, e.g., survival analysis.
Training the model in an end-to-end fashion, we demonstrate that \DGC\ is capable of capturing a clustering that aligns with the provided information, while obtaining reasonable classification results on various datasets.

\section*{Broader Impact}
\todo{This is a bit lengthy. The general sentiment is fine, but I would cut this by at least 50\% and just make it very concise and to the point.}
As hinted in the introduction, this work is motivated by the ambiguous\todo{I am really not buying this ambiguous argument. It is what you ask for. Your case might be a bit better because you start from a probabilistic model, but in general you write down a model and your clustering is then a reflection of the mode you chose. But there is nothing inherently ambiguous about. E.g., if you do k-mean clustering you know exactly what you are asking for and this is what you are going to get. CMB: perhaps ambiguity is misleading, is arbitrary a better word? you'll get different clusters and memberships even for k-means if you use an L1 distance vs an L2 distance? unlike the supervised problem, where we minimize some distance to a fixed point (the known points), our reference floats and the final distribution of subspaces is completely governed by our initial choice.} nature of clustering that is especially apparent in the field of scientific discovery, where researchers in general do not have a priori knowledge what grouping is more appropriate when multiple groupings exist\todo{Does your model really help with this?---YS: Our model is to help with this situation with an available response variable right?}. As a concrete example, in the field of breast cancer research, the problem of finding a group structure among patients that is meaningful in explaining their corresponding survival rates has attracted increasing attention in recent years. However, the final measurement, i.e. the measurement that the clustering is hoped to meaningfully explain, usually has no impact on the clustering procedure itself. Further, this approach requires a pre-defined notion of ``satisfactory clustering,'' which, in a way, is contradictory to the very concept of scientific discovery in which one would want as little human intervention as possible. In this work, we address both problems at once: we let the supervision (i.e. the final measurement) directly affect the latent space in which the clustering is performed, and more importantly, we let the network define and learn an appropriate clustering instead of pre-defining \todo{Don't you totally pre-define this notion as well as you make model assumptions?} such a notion. Therefore, we believe researchers who work in scientific discovery in general will benefit from this work. 

\section*{Acknowledgement}
Research detailed in this work was supported by the National Science Foundation (NSF) under award numbers NSF EECS-1711776 and NSF EECS-1610762.

{
\bibliographystyle{abbrv}
\bibliography{references}
}
\clearpage

\begin{appendices}
\section{Theoretical Derivations}
This section provides detailed derivations for the theoretical claims made in the main manuscript. 
\begin{theo}
The variational lower bound for {\DGC} is 
	\begin{equation} 
        \begin{split}
        \log p(\textbf{x},\textbf{y}) 
        &\geq \underbrace{\mathbb{E}_{q(\textbf{z},c|\textbf{x})}\log p(\textbf{y}|\textbf{z},c)}_{\text{Probabilistic Ensemble}} + \underbrace{\mathbb{E}_{q(\textbf{z},c|\textbf{x})}\log \frac{p(\textbf{x},\textbf{z},c)}{q(\textbf{z},c|\textbf{x})}}_\text{ELBO for \VAE\ with GMM prior} = \mathcal{L}_{\text{ELBO}}\,.
        \end{split}
    \end{equation}
\end{theo}

\begin{proof}
We derive the $\mathcal{L}_{\text{ELBO}}$ as follows
	\begin{equation}
        \begin{split}
        \log p(\textbf{x},\textbf{y}) &= \log \int_\textbf{z} \sum_c p(\textbf{x},\textbf{y},\textbf{z},c) d\textbf{z}\\
        &= \log \int_\textbf{z} \sum_c \frac{p(\textbf{x},\textbf{y},\textbf{z},c)}{q(\textbf{z},c|\textbf{x})}q(\textbf{z},c|\textbf{x}) d\textbf{z}  \\
        &\geq \mathbb{E}_{q(\textbf{z},c|\textbf{x})}\log \frac{p(\textbf{x},\textbf{y},\textbf{z},c)}{q(\textbf{z},c|\textbf{x})}\\ 
        &= \underbrace{\underbrace{\mathbb{E}_{q(\textbf{z},c|\textbf{x})}\log p(\textbf{y}|\textbf{z},c)}_{\text{Probabilistic Ensemble}} + \underbrace{\mathbb{E}_{q(\textbf{z},c|\textbf{x})}\log \frac{p(\textbf{x},\textbf{z},c)}{q(\textbf{z},c|\textbf{x})}}_\text{ELBO for VAE with GMM prior}}_{\pmb{\mathcal{L}}_{\text{ELBO}}}\,.
        \end{split}
    \end{equation}
\end{proof}

\begin{prop}
To explain the fact that choosing $q(\textbf{z}|\textbf{x})$ to be unimodal will not incur a sizable information loss when the learned $q(c|\textbf{x})$ is discriminative, we dissect the $\mathcal{L}_{\text{ELBO}}$ as follows (Eq.3 in the main paper)
\begin{equation} 
\begin{aligned}
    \mathcal{L}_{\text{ELBO}} = \mathbb{E}_{q(\textbf{z},c|\textbf{x})}\log p(\textbf{y}|\textbf{z},c) + & \mathbb{E}_{q(\textbf{z}|\textbf{x})}\log p(\textbf{x}|\textbf{z}) \\
    & - \mathbb{KL}\left(q(c|\textbf{x})||p(c)\right) - \sum_{k}\lambda_k\mathbb{KL}\left(q(\textbf{z}|\textbf{x})||p(\textbf{z}|c=k)\right)\,.
\end{aligned}
\end{equation}
\end{prop}
\begin{proof}
	\begin{equation}
        \begin{split}
        \mathcal{L}_{\text{ELBO}} &=\mathbb{E}_{q(\textbf{z},c|\textbf{x})}\log p(\textbf{y}|\textbf{z},c) +\mathbb{E}_{q(\textbf{z},c|\textbf{x})}\log \frac{p(\textbf{x},\textbf{z},c)}{q(\textbf{z},c|\textbf{x})}\\
        &= \mathbb{E}_{q(\textbf{z},c|\textbf{x})}\log p(\textbf{y}|\textbf{z},c) +\mathbb{E}_{q(\textbf{z},c|\textbf{x})}\log \frac{p(\textbf{x}|\textbf{z})p(\textbf{z}|c)p(c)}{q(\textbf{z}|\textbf{x})q(c|\textbf{x})}\\
        &= \mathbb{E}_{q(\textbf{z},c|\textbf{x})}\log p(\textbf{y}|\textbf{z},c) +\mathbb{E}_{q(\textbf{z}|\textbf{x})}\log p(\textbf{x}|\textbf{z}) - \mathbb{KL}\left(q(c|\textbf{x})||p(c)\right) +  \mathbb{E}_{q(\textbf{z},c|\textbf{x})}\log \frac{p(\textbf{z}|c)}{q(\textbf{z}|\textbf{x})}
        \end{split}
    \end{equation}
where 
$$\mathbb{E}_{q(\textbf{z},c|\textbf{x})}\log\frac{p(\textbf{z}|c)}{q(\textbf{z}|\textbf{x})} = \mathbb{E}_{q(c|\textbf{x})}\mathbb{E}_{q(\textbf{z}|\textbf{x})}\log \frac{p(\textbf{z}|c)}{q(\textbf{z}|\textbf{x})} = \sum_{k}\lambda_k\mathbb{KL}\left(q(\textbf{z}|\textbf{x})||p(\textbf{z}|c=k)\right)$$
where $\lambda_k = q(c=k|\textbf{x})$.
\end{proof}
\bigskip

\begin{prop}
Choosing $q(c|\textbf{x})$ requires us to decompose $\mathcal{L}_{\text{ELBO}}$ as follows
	\begin{equation}
        \begin{split}
        \mathcal{L}_{\text{ELBO}} &= \mathbb{E}_{q(\textbf{z},c|\textbf{x})}\log p(\textbf{y}|\textbf{z},c) + \mathbb{E}_{q(\textbf{z}|\textbf{x})}\log \frac{p(\textbf{x},\textbf{z})}{q(\textbf{z}|\textbf{x)}} -  \mathbb{E}_{q(\textbf{z}|\textbf{x})} \mathbb{KL}\left(q(c|\textbf{x})||p(c|\textbf{z})\right).
        \end{split}
    \end{equation}
\end{prop}
\begin{proof}
 	\begin{equation}
        \begin{split}
        \pmb{\mathcal{L}}_{\text{ELBO}} &= \mathbb{E}_{q(\textbf{z},c|\textbf{x})}\log p(\textbf{y}|\textbf{z},c) + \mathbb{E}_{q(\textbf{z},c|\textbf{x})}\log \frac{p(\textbf{x},\textbf{z},c)}{q(\textbf{z},c|\textbf{x})}\\
        &= \mathbb{E}_{q(\textbf{z},c|\textbf{x})}\log p(\textbf{y}|\textbf{z},c) + \mathbb{E}_{q(\textbf{z},c|\textbf{x})}\log \frac{p(c|\textbf{x},\textbf{z})p(\textbf{x},\textbf{z})}{q(\textbf{z}|\textbf{x})q(c|\textbf{x})}\\
        &= \mathbb{E}_{q(\textbf{z},c|\textbf{x})}\log p(\textbf{y}|\textbf{z},c) + \mathbb{E}_{q(\textbf{z}|\textbf{x})}\mathbb{E}_{q(c|\textbf{x})}\left[\log \frac{p(\textbf{x},\textbf{z})}{q(\textbf{z}|\textbf{x})}-\log\frac{q(c|\textbf{x})}{p(c|\textbf{z})}\right]\\
        &=\mathbb{E}_{q(\textbf{z},c|\textbf{x})} \log p(\textbf{y}|\textbf{z},c) + \mathbb{E}_{q(\textbf{z}|\textbf{x})}\log \frac{p(\textbf{x},\textbf{z})}{q(\textbf{z}|\textbf{x)}} - \mathbb{E}_{q(\textbf{z}|\textbf{x})} \mathbb{KL}\left(q(c|\textbf{x})||p(c|\textbf{z})\right)
        \end{split}
    \end{equation}
\end{proof}
\smallskip

\begin{prop}
The solution to the following convex program
\begin{equation} \label{opti_a}
\begin{aligned}
\min_{q(c|\textbf{x})} \quad & f_0(q) =  \mathbb{KL}\left(q(c|\textbf{x})||p(c|\textbf{z})\right) -  \mathbb{E}_{q(c|\textbf{x})} \log p(\textbf{y}|\textbf{z},c)\,,\\
\textrm{s.t.} \quad & \sum_k q(c=k|\textbf{x}) = 1,\quad
  q(c=k|\textbf{x}) \geq 0,\ \ \forall k\,
\end{aligned}
\end{equation}
is 
\begin{equation} 
  q(c=k|\textbf{x})  = \frac{p(\textbf{y}|\textbf{z},c=k)\cdot p(c=k|\textbf{z})}{\sum_k p(\textbf{y}|\textbf{z},c=k)\cdot p(c=k|\textbf{z})}\,. 
\end{equation}
\end{prop}
\begin{proof}
First, we note that the constraint, $q(c=k|\textbf{x}) \geq 0$ for all $k$, is not needed (and effectively redundant), as the $\mathbb{KL}$ term in the objective function is not defined otherwise. Now consider a convex program that takes the form of
\begin{equation}
\begin{aligned}
\min_{\textbf{t}\in \mathbb{R}_{+}^k} \quad & f_0(\textbf{t}) \\
\textrm{s.t.} \quad & 
  \textbf{1}^T \textbf{t} = 1\,.
\end{aligned}
\end{equation}
where $f_0$ is a convex function and ``$\succeq$'' denotes ``element-wise greater than or equal to.'' Forming the Lagrangian, we have
$$\textbf{L}\left(\textbf{t}, \gamma\right) = f_0(\textbf{t}) + \gamma \left(\textbf{1}^T \textbf{t}-1\right)$$
The \textit{Karush–Kuhn–Tucker conditions} state that the optimal solution dual, $(\textbf{t}^*, \gamma^*)$, satisfies the following
\begin{itemize}
    \item $-\textbf{t}^* \preceq 0$
    \item $\textbf{1}^T \textbf{t}^*-1 =0$
    \item $\nabla_{\textbf{t}}\textbf{L}\left(\textbf{t}^*,\gamma^*\right) = 0$
\end{itemize}
Since
$$\nabla_{\textbf{t}} \textbf{L}\left(\textbf{t}, \gamma\right) = \nabla_{\textbf{t}} f_0(\textbf{t}) + \gamma\cdot \textbf{1}$$
the third condition implies that
\begin{equation} \label{condi}
 \nabla_{\textbf{t}}\textbf{L}\left(\textbf{t}^*, \gamma^*\right)  =  \nabla_{\textbf{t}} f_0(\textbf{t}^*) + \gamma^*\cdot \textbf{1} = 0\,.   
\end{equation}
Let $\textbf{t} = q(c|\textbf{x})$ $\left(\text{i.e.}\ t_k = q(c=k|\textbf{x})\right)$, and $f_0(\textbf{t})$ as being specified in Eq.~\ref{opti_a}, we have
 	\begin{equation}
        \begin{split}
        \nabla_{t_k} f_0(\textbf{t}) &= \frac{\partial}{\partial t_k}\left(\sum_{k} t_k \log \frac{t_k}{p(c=k|\textbf{z})} - \sum_k t_k \log p(\textbf{y}|\textbf{z},c=k)\right)\\
        &= \log \frac{t_k}{p(c=k|\textbf{z})} + 1 - \log p(\textbf{y}|\textbf{z},c=k)\,.
        \end{split}
    \end{equation}
Based on the condition in Eq.~\ref{condi}, we thus have
$$\nabla_{t_k}\textbf{L}\left(\textbf{t}^*, \gamma^*\right)  = \log \frac{t^*_k}{p(c=k|\textbf{z})} + 1 - \log p(\textbf{y}|\textbf{z},c=k) + \gamma^* = 0$$
which leads to
    $$t^*_k = e^{\log p(\textbf{y}|\textbf{z},c=k) - 1 - \gamma^*} \cdot p(c=k|\textbf{z})\,.$$

Since $\gamma^*$ is chosen in a way such that $\sum_k t^*_k = 1$ (by the second condition), we obtain the solution 
\begin{equation} \label{solu}
    t^*_k = \frac{t^*_k}{\sum_k t^*_k} = \frac{p(\textbf{y}|\textbf{z},c=k)\cdot p(c=k|\textbf{z})}{\sum_k p(\textbf{y}|\textbf{z},c=k)\cdot p(c=k|\textbf{z})}\,.
\end{equation}
\end{proof}

\section{Experimental Details}
This section provides a detailed description of the experimental setups, such as the train/test splits, the chosen network architectures, and the choices of learning rate and optimizer, for the experiments conducted. We describe the architecture of {\DGC} in terms of its encoder, decoder, and task network. We adopt the following abbreviations for some basic network layers
\begin{itemize}
    \item \FL($d_i,d_o,f$) denotes a fully-connected layer with $d_i$ input units, $d_o$ output units, and activation function $f$.
    \item \Conv$\left(c_i, c_o, k_1, f, \batch, O(k_2,s) \right)$ denotes a convolution layer with $c_i$ input channels, $c_o$ output channels, kernel size $k_1$, activation function $f$, and pooling operation $O(k_2,s)$ with another kernel size $k_2$ and stride $s$.
\end{itemize} 

\subsection{Noisy MNIST}
We extract images that correspond to the digits 2 and 7 from MNIST. The MNIST dataset is pre-divided into training/testing sets,  so we naturally use the images that correspond to the digits 2 and 7 from the training set as our training data (12,223 images), and that from the testing set as our testing data (2,060 images). For each digit, we randomly select half of the images for that digit and superpose noisy backgrounds onto those images, where the backgrounds are cropped from randomly selected CIFAR-10 images (more specifically, we first randomly select a class, and then randomly select a CIFAR image that corresponds to that class). 

We use the \texttt{Adam} optimizer for optimization. We train with a batch size of 128 images, an initial learning rate of 0.002, and a learning rate decay of 10\% after every 10 epochs, for 100 epochs. 
We use the following network architecture:
\begin{center}
 \label{tab:MNIST}
 \begin{tabular}{||c||} 
 \hline
 \textbf{Encoder}  \\ [0.5ex] 
 \hline\hline
  \FL(784,500,\ReLU) \\ 
 \hline
  \FL(500,500,\ReLU) \\ 
 \hline
  \FL(500,2000,\ReLU) \\
 \hline
  \FL(2000,10,\ReLU) \\
 \hline
\end{tabular}
\quad
 \begin{tabular}{||c||} 
 \hline
 \textbf{Decoder}  \\ [0.5ex] 
 \hline\hline
  \FL(10, 2000,\ReLU) \\ 
 \hline
  \FL(200,500,\ReLU) \\ 
 \hline
  \FL(500,500,\ReLU) \\
 \hline
  \FL(500,784,\ReLU) \\
 \hline
\end{tabular}
\quad
 \begin{tabular}{||c||} 
 \hline
 \textbf{Task Network}  \\ [0.5ex] 
 \hline
  \FL(10, 4,\sig) \\ 
 \hline
\end{tabular}
\end{center}

\subsection{Pacman}

This section provides more details for our Pacman experiments. 

\begin{figure}[h]
    \centering 
\begin{subfigure}{0.35\textwidth}
  \includegraphics[width=1.3\linewidth]{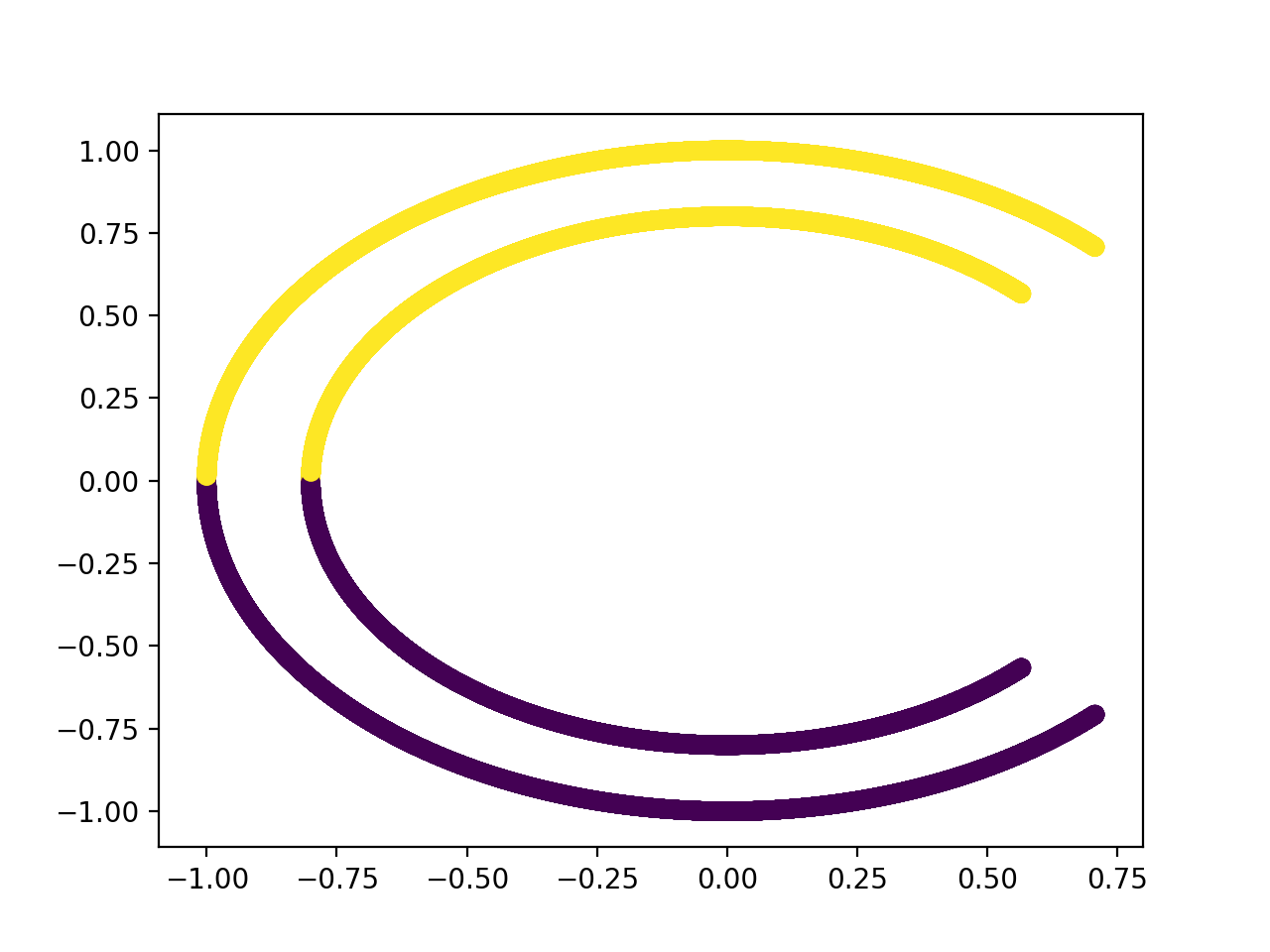}
  \caption{K-Means Clustering Result}
  \label{fig:1a}
\end{subfigure}\hfil 
\begin{subfigure}{0.35\textwidth}
  \includegraphics[width=1.3\linewidth]{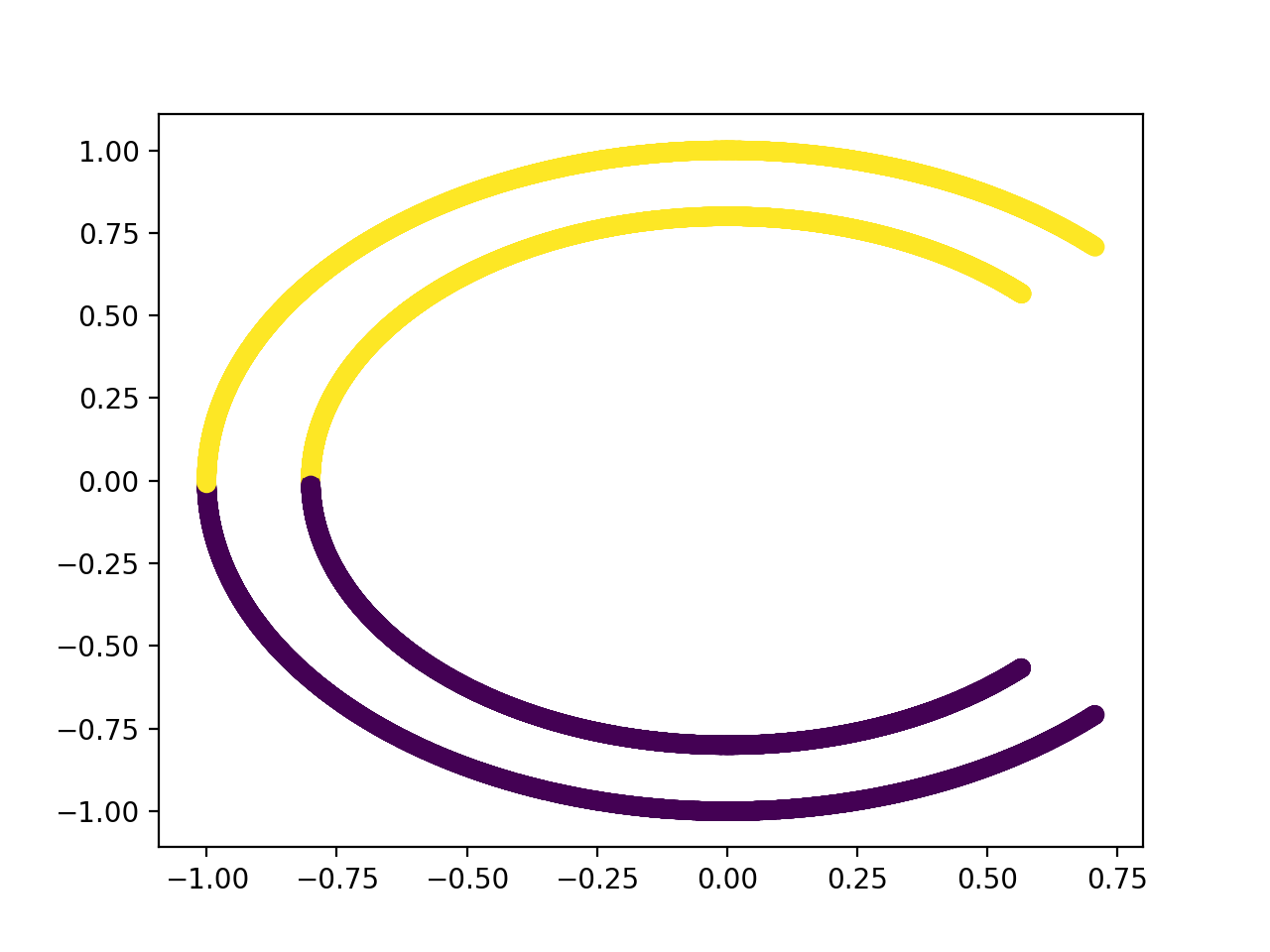}
  \caption{Spectral Clustering Result}
  \label{fig:2a}
\end{subfigure}\hfil 

\medskip
\begin{subfigure}{0.35\textwidth}
  \includegraphics[width=1.3\linewidth]{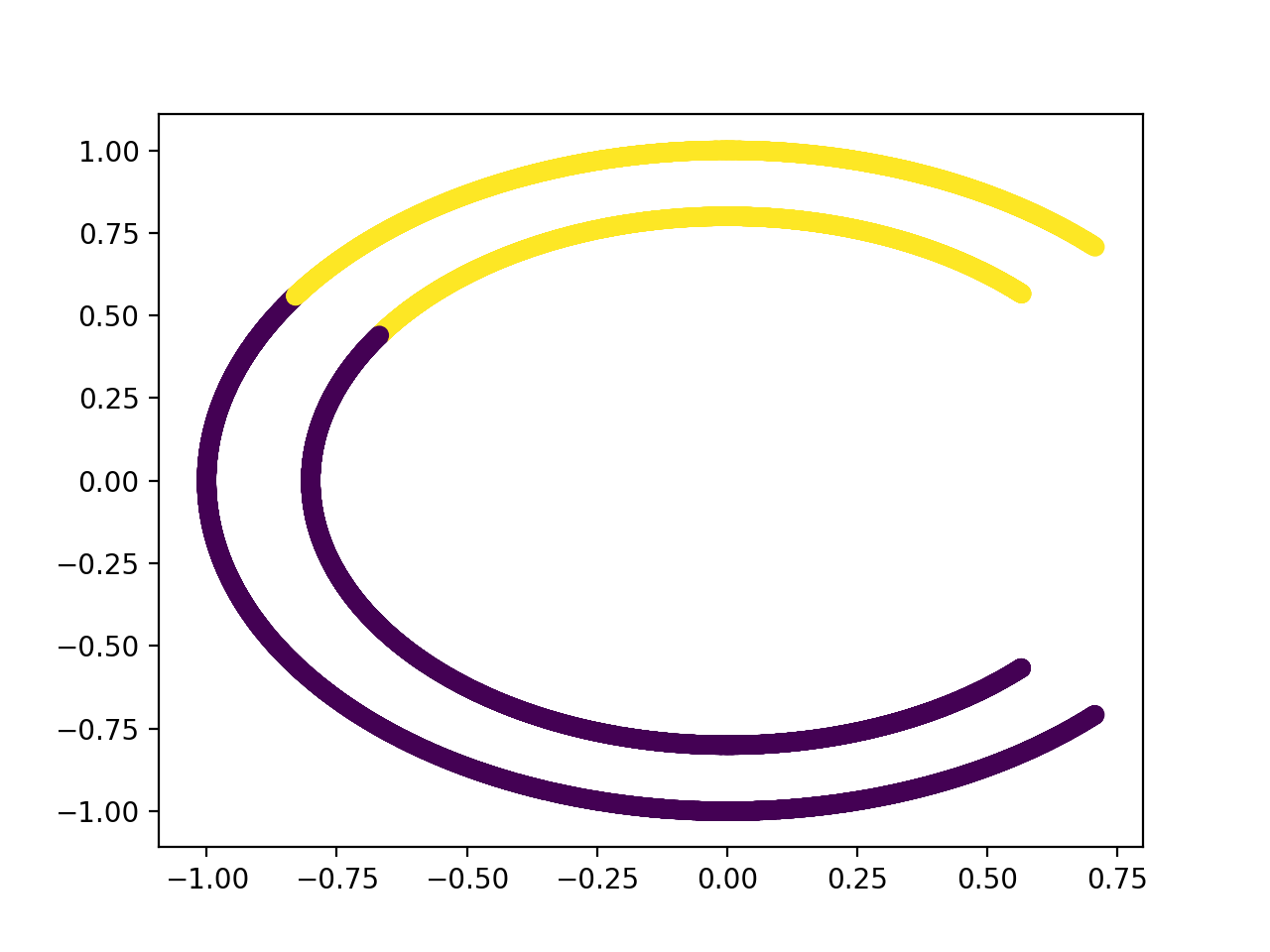}
  \caption{Hierarchical Clustering Result}
  \label{fig:3a}
\end{subfigure}\hfil 
\begin{subfigure}{0.35\textwidth}
  \includegraphics[width=1.3\linewidth]{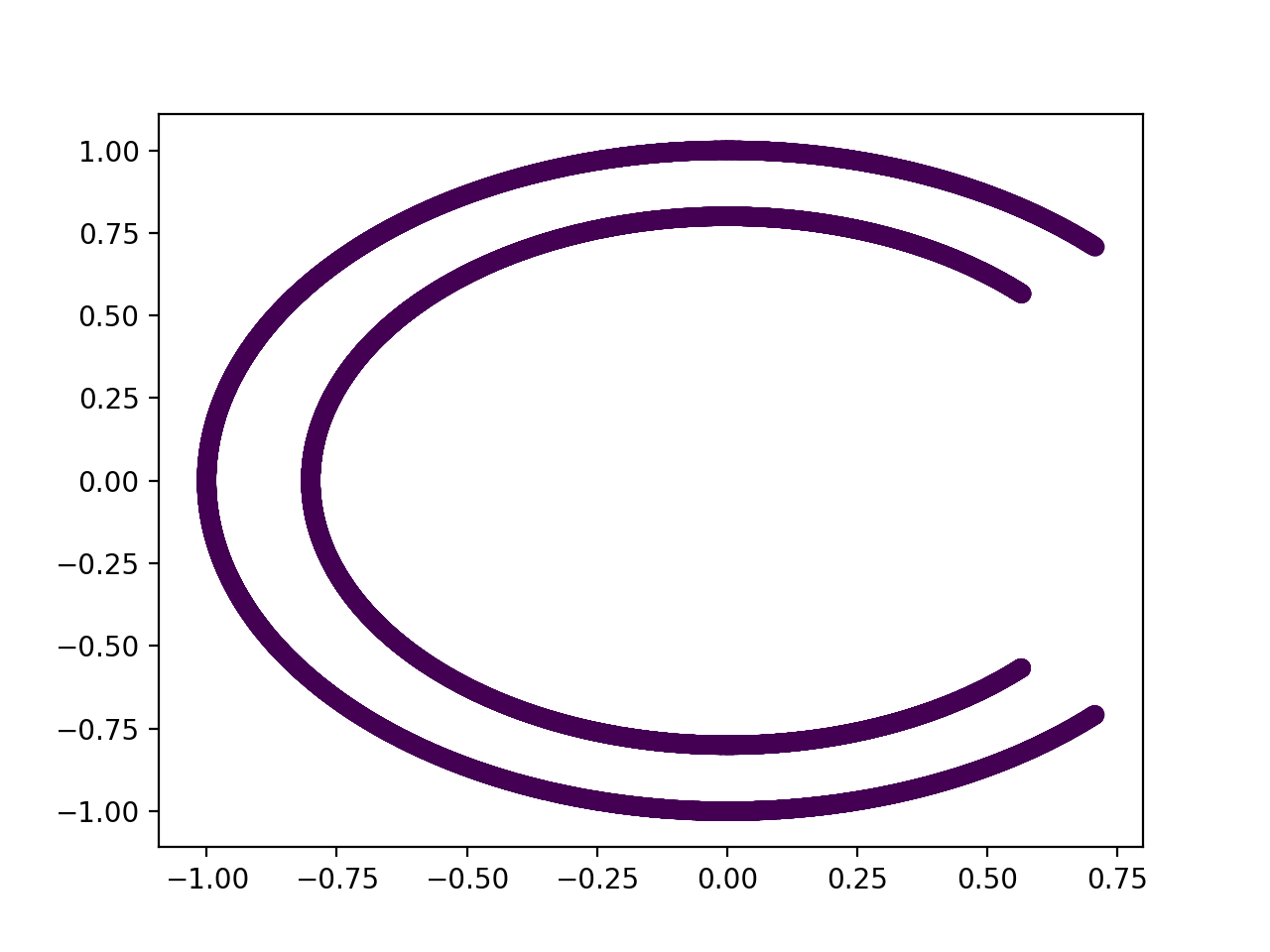}
  \caption{{\VaDE} Clustering Result}
  \label{fig:4a}
\end{subfigure}\hfil 
\caption{Clustering results obtained from four different unsupervised clustering methods, namely (a) K-means clustering; (b) spectral clustering; (c) hierarchical clustering; and (d) {\VaDE}.}
\label{fig:pacman_bad}
\end{figure}

\subsubsection{Experimental Setup}

We create 20,000 points, with 10,000 for the outer annulus and 10,000 for the inner annulus. Both annuli center at the origin, with the outer annulus having a radius of 1 and the inner annulus having a radius of 0.8. We create the training set by sampling 7,500 points from each annulus, and leave the rest of the data for testing. We create the linear responses by dividing the [0,1] range into 10,000 sub-intervals and assign the split points to the points in the inner annulus in a way that it is increasing (from 1 to 0) in the clockwise direction. We create the exponential responses by evaluating the exponential function at the aforementioned split points (generated for the linear responses), and then assign them to the points on the outer annulus in a way that it is decreasing (from 0 to 1) in the clockwise direction.

We use the \texttt{Adam} optimizer for optimization. We train with a batch size of 1,000 points, an initial learning rate of 0.001, and a learning rate decay of 10\% after every 10 epochs, for 80 epochs. 
We use the following network architecture:
\begin{center}
 \label{tab:pacman_a}
 \begin{tabular}{||c||} 
 \hline
 \textbf{Encoder}  \\ [0.5ex] 
 \hline\hline
  \FL(2,64,\sig) \\ 
 \hline
  \FL(64,128,\sig) \\ 
 \hline
  \FL(128,256,\sig) \\
 \hline
  \FL(256,60,\sig)\\
 \hline
\end{tabular}
\quad
 \begin{tabular}{||c||} 
 \hline
 \textbf{Decoder}  \\ [0.5ex] 
 \hline\hline
  \FL(60, 256,\sig) \\ 
 \hline
  \FL(256,128,\sig) \\
 \hline
  \FL(128,64,\sig) \\
 \hline
  \FL(64,2,\sig)\\
 \hline
\end{tabular}
\quad
 \begin{tabular}{||c||} 
 \hline
 \textbf{Task Network}  \\ [0.5ex] 
 \hline\hline
  \FL(64, 128,\sig) \\ 
 \hline
  \FL(128,4,\sig) \\ 
 \hline
\end{tabular}
\end{center}

\subsubsection{Attempts using Unsupervised Methods}

As mentioned in the main manuscript, we have tried the following unsupervised methods on the Pacman dataset to see how they perform: K-means clustering, hierarchical clustering, spectral clustering, and {\VaDE}. Fig.~\ref{fig:pacman_bad} shows the clustering results. None of these methods can separates the two annuli in a satisfactory way. This phenomenon echos a deep-rooted obstacle for clustering methods in general: the concept of clusters is inherently subjective, and different clustering methods can potentially produce very different clustering results. Furthermore, for most unsupervised clustering methods, it is nontrivial to incorporate a prior given information about the clusters, even when such information exists. Case in point, even when provided the information that the goal is to separate the two annuli, it is not clear how to incorporate that information into any of the unsupervised methods presented here.

\subsection{SVHN}
We apply \DGC\ to the Street View House Number (SVHN) dataset. This dataset consists of 73,257 training images, 26,032 test images, and 531,131 additional training images (that are easier than the ones in the training set). We train {\DGC} using all the training and extra images. 

We use the \texttt{Adam} optimizer for optimization. We train with a batch size of 128 images and a learning rate of 0.0001 (stays constant throughout epochs) for 150 epochs. 
We use the following network architecture:
\begin{center}
 \label{tab:svhn_a}
 \begin{tabular}{||c||} 
 \hline
 \textbf{Encoder}  \\ [0.5ex] 
 \hline\hline
  \Conv$\left(3, 48, 5, \ReLU,\batch, \maxp(2,2) \right)$ \\ 
 \hline
  \Conv$\left(48, 64, 5, \ReLU,\batch, \maxp(2,1) \right)$ \\ 
 \hline
  \Conv$\left(64, 128, 5, \ReLU,\batch, \maxp(2,2) \right)$ \\
 \hline
  \Conv$\left(128, 160, 5, \ReLU,\batch, \maxp(2,1) \right)$\\
   \hline
  \Conv$\left(160, 192, 5, \ReLU,\batch, \maxp(2,2) \right)$\\
  \hline
  \Conv$\left(192, 192, 5, \ReLU,\batch, \maxp(2,1) \right)$\\
    \hline
  \Conv$\left(192, 192, 5, \ReLU,\batch, \maxp(2,2) \right)$\\
    \hline
  \Conv$\left(192, 192, 5, \ReLU,\batch, \maxp(2,1) \right)$\\
    \hline
  \FL(9408,3072, \ReLU)\\
    \hline
  \FL(3072,256, \ReLU)\\
 \hline
\end{tabular}

 \begin{tabular}{||c||} 
 \hline
 \textbf{Decoder}  \\ [0.5ex] 
 \hline\hline
  \FL(256, 3072,\ReLU) \\ 
 \hline
  \Conv$\left(3072, 256, 4, \ReLU,\batch, \maxp(2,2) \right)$ \\
 \hline
  \Conv$\left(256, 128, 4, \ReLU,\batch, \maxp(2,1) \right)$ \\
 \hline
  \Conv$\left(128, 64, 4, \ReLU,\batch, \maxp(2,2) \right)$\\
 \hline
 \Conv$\left(64, 3, 4, \ReLU,\batch, \maxp(2,1) \right)$\\
 \hline
\end{tabular}
\quad
 \begin{tabular}{||c||} 
 \hline
 \textbf{Task Network}  \\ [0.5ex] 
 \hline\hline
  \FL(256, 512, \sig) \\ 
 \hline
  \FL(512,1024, \sig) \\ 
 \hline
   \FL(1024,512, \sig) \\ 
 \hline
   \FL(512,256, \sig) \\ 
 \hline
   \FL(256,100, \sig) \\ 
 \hline
\end{tabular}
\end{center}
\end{appendices}

\end{document}